\documentclass{article}
\pdfoutput=1

\usepackage{amsmath}
\usepackage{amssymb}
\usepackage{bm}
\usepackage{color}
\usepackage[usenames,dvipsnames]{xcolor}
\usepackage{dsfont}
\usepackage[top=1in, bottom=1in, left=1in, right=1in]{geometry}
\usepackage{graphicx}
\usepackage{multirow}
\usepackage{natbib}
\usepackage[caption=false]{subfig}
\usepackage{url}

\newcommand{\cvec}{{\bm c}}
\newcommand{\fvec}{{\bm f}}
\newcommand{\lmat}{{\bm L}}
\newcommand{\wmat}{{\bm W}}
\newcommand{\dmat}{{\bm D}}

\newcommand{\rvec}{{\bm r}}

\newcommand{\amat}{{\bm A}}

\newcommand{\gmat}{{\bm G}}

\newcommand{\lzeronorm}[1]{\left|\left|#1\right|\right|_0}
\newcommand{\lonenorm}[1]{\left|\left|#1\right|\right|_1}
\newcommand{\ltwonorm}[1]{\left|\left|#1\right|\right|_2}

\newcommand{\scones}{SConES}
\newcommand{\glasso}{non-overlapping group Lasso}

\newtheorem{proposition}{Proposition}
\newtheorem{proof}{Proof}

\DeclareMathOperator*{\argmin}{arg\,min}
\DeclareMathOperator*{\argmax}{arg\,max}

\title{Efficient network-guided multi-locus association mapping with
graph cuts}

 \author{Chlo\'e-Agathe Azencott\footnote{\texttt{chloe-agathe.azencott@tuebingen.mpg.de}}~\footnote{Machine Learning and Computational Biology Research Group,  Max Planck Institute for Developmental Biology \& 
     Max Planck Institute for Intelligent Systems 
     Spemannstr. 38, 72076 T\"ubingen, Germany}~, 
   Dominik Grimm$^\dagger$,
   Mahito Sugiyama$^\dagger$,
   Yoshinobu Kawahara\footnote{The Institute of Scientific  and Industrial Research (ISIR)
     Osaka University
     8-1 Mihogaoka, Ibaraki-shi, Osaka 567-0047 Japan} \\and
   Karsten M. Borgwardt$^\dagger$\footnote{Zentrum f\"ur Bioinformatik, Eberhard Karls Universit\"at
   T\"ubingen, 72076 T\"ubingen, Germany} 
 }

\begin{document}

 \maketitle

\begin{abstract}
    As an increasing number of genome-wide association studies reveal
    the limitations of attempting to explain phenotypic heritability
    by single genetic loci, there is growing interest for associating
    complex phenotypes with {\it sets of genetic loci}.  While several
    methods for multi-locus mapping have been proposed, it is often
    unclear how to relate the detected loci to the growing knowledge
    about gene pathways and networks.  The few methods that take
    biological pathways or networks into account are either restricted
    to investigating a limited number of predetermined sets of loci,
    or do not scale to genome-wide settings.

    We present \scones, a new efficient method to discover sets of
    genetic loci that are maximally associated with a phenotype, while
    being connected in an underlying network.  Our approach is based
    on a minimum cut reformulation of the problem of selecting
    features under sparsity and connectivity constraints, which can be
    solved exactly and rapidly.

    \scones\ outperforms state-of-the-art competitors in terms of
    runtime, scales to hundreds of thousands of genetic loci, and
    exhibits higher power in detecting causal SNPs in simulation
    studies than existing methods. On flowering time phenotypes and
    genotypes from {\it Arabidopsis thaliana}, \scones\ detects loci
    that enable accurate phenotype prediction and that are supported
    by the literature.

    Matlab code for \scones\ is available at
    \url{http://webdav.tuebingen.mpg.de/u/karsten/Forschung/scones/}.
\end{abstract}

\section{Introduction}
\label{sec:intro}

Twin and family/pedigree studies make it possible to estimate the
heritability of observed traits, that is to say the amount of their
variability that can be attributed to genetic differences.  In the
past few years, genome-wide association studies (GWAS), in which
several hundreds of thousands to millions of single nucleotide
polymorphisms (SNPs) are assayed in up to thousands of individuals,
have made it possible to identify hundreds of genetic variants
associated with complex phenotypes~\citep{zuk12}.  Unfortunately,
while studies associating single SNPs with phenotypic outcomes have
become standard, they often fail to explain much of the heritability
of complex traits~\citep{manolio09}.  Investigating the joint effects
of multiple loci by mapping sets of genetic variants to the phenotype
has the potential to help explain part of this missing
heritability~\citep{marchini05}.  While efficient multiple linear
regression approaches~\citep{cho10,wang11,rakitsch12} make the
detection of such multivariate associations possible, they often
remain limited in power and hard to interpret.  Incorporating
biological knowledge into these approaches could help boosting their
power and interpretability.  However, current methods are limited to
predefining a reasonable number of candidate sets to
investigate~\citep{cantor10,fridley11,wu11}, for instance by relying
on gene pathways.  They consequently run the risk of missing
biologically relevant loci that have not been included in the
candidate sets.  This risk is made even likelier by the incomplete
state of our current biological knowledge.

For this reason, our goal here is to use prior knowledge in a more
flexible way. We propose to use a biological network, defined between
SNPs, to guide a multi-locus mapping approach that is both efficient
to compute and biologically meaningful: {\it We aim to find a set of
  SNPs that (a) are maximally associated with a given phenotype and
  (b) tend to be connected in a given biological network. In addition,
  this set must be computed efficiently on genome-wide data}.  In this
paper we assume an additive model to characterize multi-locus
association.  The network constraint stems from the assumption that
SNPs influencing the same phenotype are biologically linked. However,
the diversity of the type of relationships that this can encompass,
together with the current incompleteness of biological knowledge,
makes providing a network in which all the relevant connections are
present unlikely. For this reason, while we want to encourage the SNPs
to form a subnetwork of the network, we also do not want to enforce
that they \emph{must} form a single connected component.  Finally, we
stress that the method must scale to networks of hundreds of thousands
or millions of nodes. Approaches by~\citet{nacu07},~\citet{chuang07}
or~\citet{li08} developed to analyze gene networks containing hundreds
of nodes do therefore not apply.

While our method can be applied to any network between genetic
markers, we explore three special types of networks (
Figure~\ref{fig:networks}):
\begin{itemize}
        \item \emph{Genomic sequence network} (GS): SNPs adjacent on
    the genomic sequence are linked together. In this setting we aim
    at recovering sub-sequences of the genomic sequence that correlate
    with the phenotype.
        \item \emph{Gene membership network} (GM): SNPs are connected
    as in the sequence network described above; in addition, SNPs near
    the same gene are linked together as well. Usually, a SNP is
    considered to belong to a gene if it is either located inside said
    gene ore within a pre-defined distance of this gene. In this
    setting we aim more particularly at recovering genes that
    correlate with the phenotype.
        \item \emph{Gene interaction network} (GI): SNPs are connected
    as in the gene membership network described above. In addition,
    supposing we have a gene-gene interaction network (derived, for
    example, from protein-protein interaction data or gene expression
    correlations), SNPs belonging to two genes connected in the gene
    network are linked together. In this setting, we aim at recovering
    potential pathways that explain the phenotype.
\end{itemize}

Our task is a feature selection problem in a graph-structured feature
space, where the features are the SNPs and the selection criterion
should be related to their association with the phenotype considered.
Note that our problem is different from subgraph selection problems
such as those encountered in chemoinformatics, where each object is a
graph and each feature is a subgraph of its
own~\citep{tsuda_graph_2011}.

Several approaches have already been developed for selecting
graph-structured features.  A number of them~\citep{lesaux05,jie12}
only use the graph over the features to build the learners evaluating
their relevance, but do not enforce that the selected features should
follow this underlying structure. Indeed they can be applied to
settings where the features connectivity varies across examples, while
here all individuals share the same network.

\begin{figure}[t]
    \vspace{-10pt} \centering \subfloat[Genomic sequence network: SNPs
    adjacent on the genomic sequence are connected to each other.]{
      \includegraphics[width=0.42\linewidth]{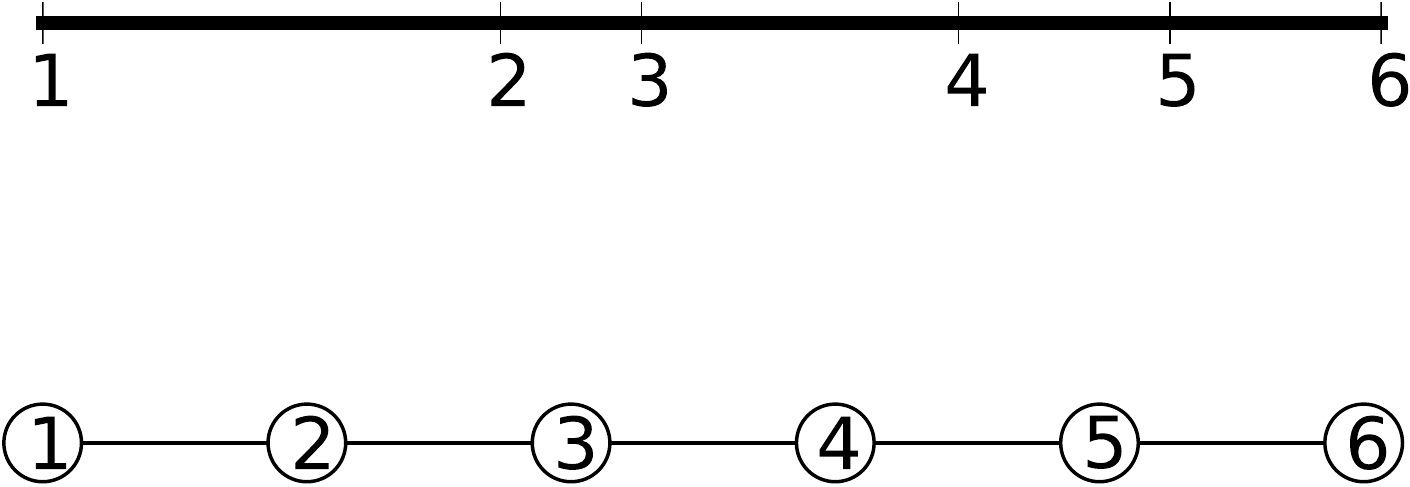}
      \label{fig:networks-basic}} \qquad \subfloat[Gene membership
    network: In addition, SNPs near the same gene (within a specified
    distance) are connected.]{
      \includegraphics[width=0.42\linewidth]{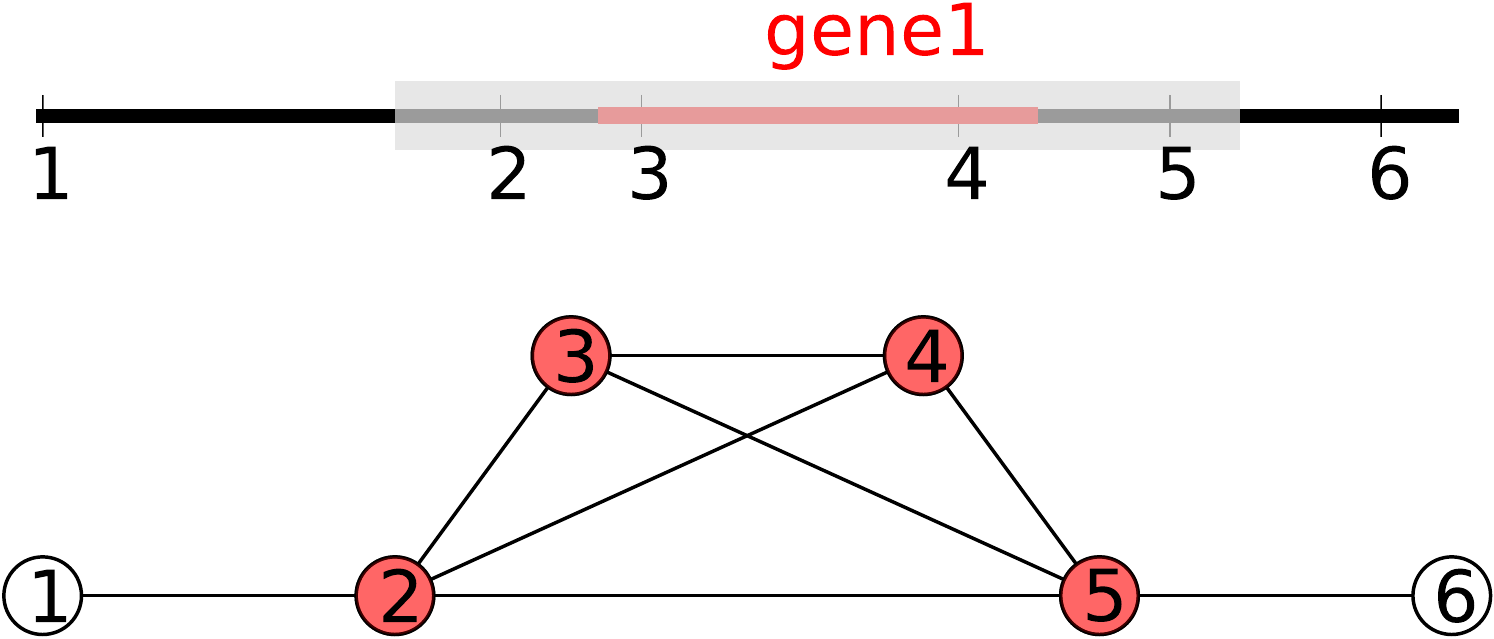}
      \label{fig:networks-gene}
    } \\
    \subfloat[Gene-interaction network: In addition, SNPs near two
    interacting genes are connected.]{
      \includegraphics[width=0.5\linewidth]{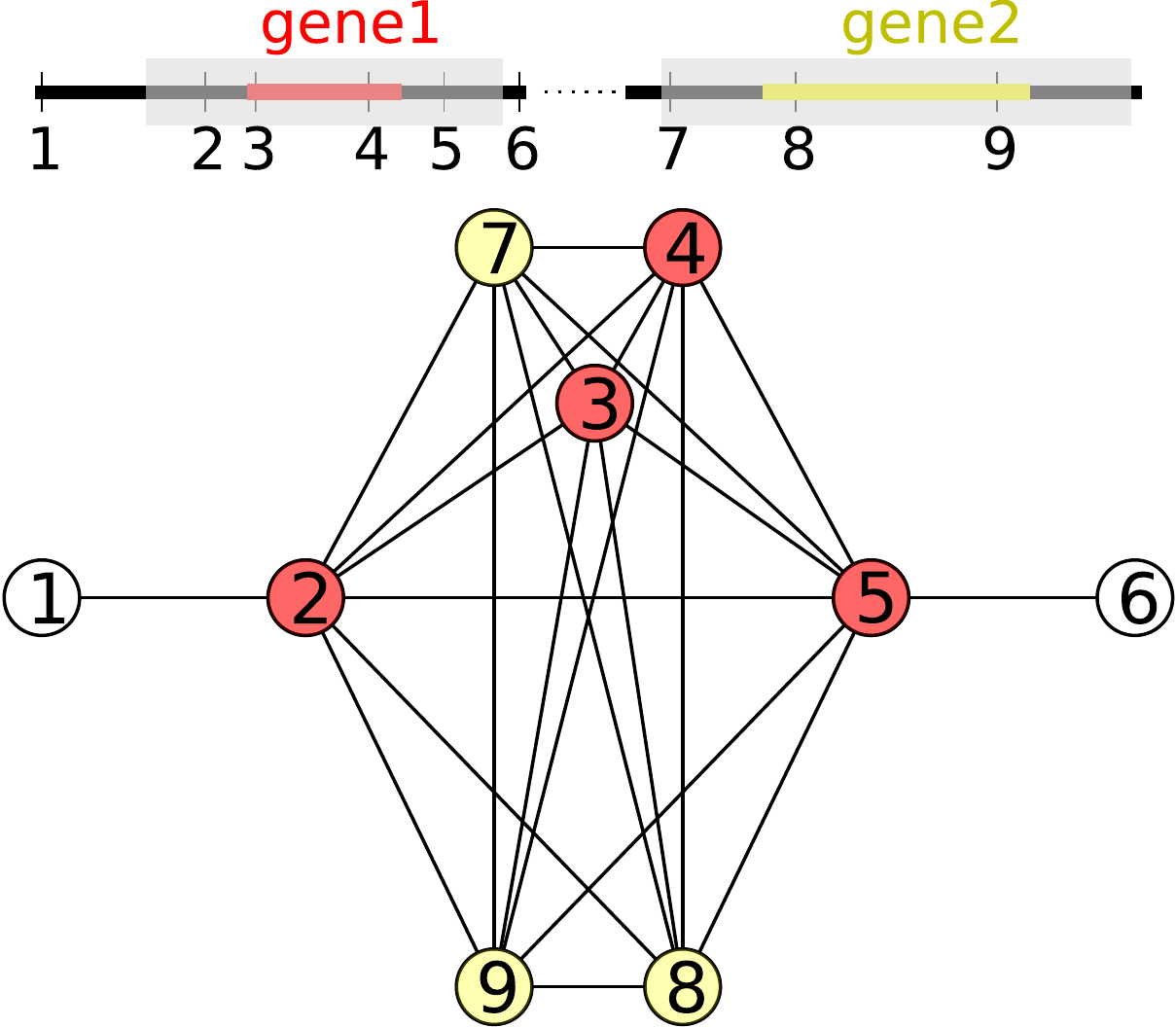}
      \label{fig:networks-interaction}
    }
    \caption{Small examples of the three types of networks
      considered.}
    \label{fig:networks}
\end{figure}

The overlapping group Lasso~\citep{jacob09,liu12} is a sparse linear
model designed to select features that belong to the union of a small
number of predefined groups.  If a graph over the features is given,
defining those groups as all pairs of features connected by an edge or
as all linear subgraphs of a given size yields the so-called graph
Lasso.  A similar approach is taken by~\citet{huang09}: their
structured sparsity penalty encourages selecting a small number of
base blocks, where blocks are sets of features defined so as to match
the structure of the problem. In the case of a graph-induced
structure, blocks are defined as small connected components of that
graph.  As shown in~\citet{mairal11}, the overlapping group Lasso
mentioned above is a relaxation of this binary problem.  As the number
of linear subgraphs or connected components of a given size grows
exponentially with the number of nodes of the graph, which can reach
millions in the case of whole genome SNP data, only the edge-based
version of the graph Lasso can be applied to our problem. It is
however unclear whether it is sufficient to capture long-range
connections between graph nodes.

\citet{li08} propose a network-constrained version of the Lasso that
imposes the type of graph connectivity we deem desirable.  However,
their approach has been developed with networks of genes (rather than
of SNPs) in mind and does not scale easily to the data sets we
envision.  Indeed, the implementation they propose relies on a
singular value decomposition of the Laplacian of the network, which is
intensive to compute and cannot be stored in memory.

\citet{chuang07} also searched subnetworks of protein-protein
interaction networks that are maximally associated with a phenotype;
however, their greedy approach requires to fix beforehand a
(necessarily small) upper-limit on the size of the subnetworks
considered.

In the case of directed acyclic graphs,~\citet{mairal11} propose a
minimum flow formulation that make it possible to use for groups (or
blocks) the set of all paths of the network. Unfortunately, the
generalization to undirected graphs with cycles, such as the SNP
networks we consider, requires to randomly assign directions to edges
and prune those in cycles without any biological justification.
Although this can work reasonably well in practice~\citep{mairal11},
this is akin to artificially removing more than half of the network
connections without any biological justification.

In what follows, we formulate the network-guided SNP selection problem
as a minimum cut problem on a graph derived from the SNP network in
Section~\ref{sec:methods} and evaluate the performance of our solution
both in simulations and on actual \emph{Arabidopsis thaliana} data in
Section~\ref{sec:results}.

\section{Methods}
\label{sec:methods}

\subsection{Problem Formulation}
\label{sec:problem-formulation}

Let $n$ be the number of SNPs and $m$ the number of individuals. The
SNP-SNP network is described by its adjacency matrix $\wmat$ of size
$n \times n$.  A number of statistics based on covariance matrices,
such as HSIC~\citep{gretton05} or SKAT~\citep{wu11}, can be used to
compute a measure of dependence $\cvec\in\mathbb{R}^n$ between each
single SNP and the phenotype.  Under the common assumption that the
joint effect of several SNPs is additive (which corresponds to using
linear kernels in those methods), $\cvec$ is such that the association
between a group of SNPs and the phenotype can be quantified as the sum
of the scores of the SNPs belonging to this group. That is, given an
indicator vector $\fvec \in \{0, 1\}^n$ such that, for any $p \in \{1,
\cdots, n\}$, $f_p$ is set to $1$ if the $p$-th SNP is selected and
$0$ otherwise, the score of the selected SNPs is given by $Q(\fvec) =
\sum_{p=1}^n c_p f_p = \cvec ^\top \fvec$.

We want to find the indicator vector $\fvec$ that maximizes $Q(\fvec)$
while ensuring that the solution is made of connected components of
the SNP network.  However, in general, it is difficult to find a
subset of SNPs that satisfies the above two properties. In fact, given
a positive integer $k$, the problem of finding a connected subgraph
with $k$ vertices that maximize the sum of the weights on the
vertices, which is equivalent to $Q(\fvec)$ of our case, is known to
be a strongly {\bf NP}-complete problem~\citep{lee96}. Therefore, this
problem is often addressed based on enumeration-based algorithms,
whose runtime grows exponentially with $k$. To cope with this problem,
we consider an approach based on a graph-regularization scheme, which
allows us to drastically reduce the runtime.

\subsection{Feature Selection with Graph Regularization}
\label{sec:maximum-weight-score}
Rather than searching through all subgraphs of a given network, we
reward the selection of adjacent features through graph
regularization. As it is also desirable for biological interpretation
and to avoid selecting large number of SNPs in linkage disequilibrium,
that the selected sub-networks are small in size, we reward sparse
solutions. The first requirement can be addressed by means of a
smoothness regularizer on the network~\citep{smola03,ando07}, while
the second one can be enforced with an $l_0$ constraint:
\begin{equation}
    \argmax_{\fvec \in \{0, 1\}^n}~ \underbrace{\cvec ^\top
      \fvec}_\textrm{association} - 
    \underbrace{\lambda~ \fvec ^\top \lmat
      \fvec}_\textrm{connectivity} - 
    \underbrace{\eta~ \lzeronorm{\fvec}}_\textrm{sparsity}
    \label{eq:pb1}
\end{equation}
where $\lmat$ is the Laplacian of the SNP network. $\lmat$ is defined
as $\lmat = \dmat - \wmat$, where $\dmat$ is the diagonal matrix where
$\dmat_{p,p}$ is the degree of node $p$.  Note that here, we directly
minimize the number of non-zero entries in $f$ and do not require the
proxy of an $l_1$ constraint to achieve sparsity (of course in the
case of binary indicators, $l_1$ and $l_0$ norms are equivalent).
Positive parameters $\lambda$ and $\eta$ control the importance of the
connectedness of selected features and the sparsity regularizer,
respectively.

Since $W_{p,q} = 1$ if $q$ is a neighbor of $p$ (also written as $p
\sim q$), and $0$ otherwise, if we denote by $\mathcal{N}(p)$ the
neighborhood of $p$, then the degree of $p$ can be rewritten $D_{p,p}
= \sum_{q \in \mathcal{N}(p)} 1$.  The second term in
Eq.~\eqref{eq:pb1} can therefore be rewritten as
\begin{equation}
    \fvec ^\top \lmat \fvec = \sum_{p \sim q}  (f_p - f_q)^2,
    \label{eq:term2}
\end{equation}
and the problem in Eq.~\eqref{eq:pb1} is equivalent to
\begin{equation}
    \argmax_{\fvec \in \{0, 1\}^n}~ \sum_{p=1}^{n} f_p 
    (c_p - \eta) -
    \lambda \sum_{p \sim q}  (f_p - f_q)^2 \;  .
    \label{eq:pb2}
\end{equation}
As $(f_p -f_q)^2$ is $1$ if $f_p \neq f_q$ and $0$ otherwise, it can
be seen that the connectivity term in Eq.~\eqref{eq:pb1} penalizes the
selection of SNPs not connected to one another, as well as the
selection of only subnetworks of connected components of the SNP
network.  Note that it does not prohibit the selection of several
disconnected subnetworks. In particular, solutions may include
individual SNPs fully disconnected from the other selected SNPs.
Also, as $\lzeronorm{\fvec} = \mathds{1}_n ^\top \fvec$ in our case,
the sparsity term in Eq.~\eqref{eq:pb1} is equivalent to reducing the
individual SNP scores $\cvec$ by a constant $\eta > 0$.

\subsection{Min-Cut Solution}
\label{sec:min-cut-solution}

\begin{figure}[t]
    \centering
    \includegraphics[width=.8\textwidth]{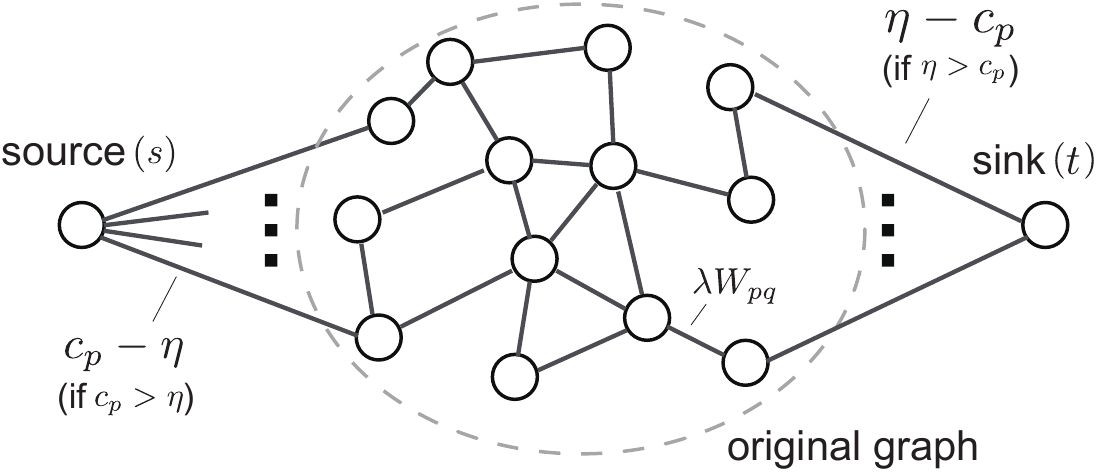}
    \caption{Graph for the $s$/$t$-min-cut formulation of the
      selection of networks of genetic markers.}
    \label{fig:stcut}
\end{figure}

A {\em cut} on a weighted graph over vertices $V := \{1, \dots, n\}$
is a partition of $V$ in a non-empty set $S$ and its complementary $V
\setminus S$. The {\em cut-set} of the cut is the set of edges whose
end vertices belong to different sets of the partition.  The {\em
  minimum cut} of the graph is the cut such that the sum of the
weights of the edges belonging to its cut-set is minimum.  If $\amat$
is the adjacency matrix of the graph, finding the minimum cut is
equivalent to finding $S \subset V$ that minimizes the {\em
  cut-function} $\sum_{p \in S} \sum_{q \notin S} \amat_{p,q} =
\sum_{p=1}^n\sum_{q=1}^n \fvec_p (1-\fvec_q) \amat_{p,q}$ where
$\fvec_p$ is $1$ if $p \in S$ and $0$ otherwise.  Given two vertices
$s$ and $t$, an {\em $s$/$t$-cut} is a cut such that $s \in S$ and $t
\in V \setminus S$.  According to the max-flow min-cut
theorem~\citep{papadimitriou82}, a minimum $s$/$t$-cut can be
efficiently computed with the maximum flow
algorithm~\citep{goldberg88}.

\begin{proposition}
    \label{prop}
    Given a graph $\mathcal{G}$ of adjacency matrix $\wmat$, solving
    the graph-regularized feature selection problem formalized in
    Eq.~\eqref{eq:pb1} is equivalent to finding an $s$/$t$ min-cut on
    the graph, depicted in Figure~\ref{fig:stcut}, whose vertices are
    that of $\mathcal{G}$, augmented by two additional nodes $s$ and
    $t$, and whose edges are given by the adjacency matrix $\amat$,
    where $\amat_{p,q}= \lambda \wmat_{p,q}$ for $1 \le p,q \le n$ and
    \setlength{\arraycolsep}{0pt}
    \begin{equation*}
        \begin{split}
            \amat_{s,p} = \left\{ \begin{array}{cl}
                    c_p - \eta & \mbox{ if } c_p > \eta \\
                    0 & \mbox{ otherwise} \\
                \end{array} \right.
            \text{ and }
            \amat_{t,p} = \left\{ \begin{array}{cl}
                    \eta - c_p & \mbox{ if } c_p < \eta \\
                    0 & \mbox{ otherwise}  \;  \\
                \end{array} \right.\\
            \hfill (p=1,\ldots,n).
        \end{split}
    \end{equation*}
\end{proposition}
\begin{proof}
    The problem in Eq.~\eqref{eq:pb1} is equivalent to
    \begin{equation}
        \argmin_{\fvec \in \{0, 1\}^n} ~(\eta \mathds{1}_n - \cvec) ^\top \fvec + \lambda \fvec^\top\lmat\fvec \,  .
        \label{eq:pb3}
    \end{equation}
    The second term of the objective is a cut-function over
    $\mathcal{G}$:
    \begin{equation*}
        \fvec^\top\lmat\fvec =
        \sum_{p=1}^nf_p\left(D_{p,p}-\sum_{q=1}^nW_{p,q}f_q\right) =
        \sum_{p=1}^n
        \sum_{q=1}^n W_{p,q}f_p(1-f_q).
    \end{equation*}  
    The first term can also be encoded as a cut-function by
    introducing to artificial nodes $s$ and $t$:
    \begin{equation*}
        \arraycolsep=1pt 
        \begin{array}{rcl}
            \displaystyle \sum_{p=1}^n (\eta - c_p) \fvec_p  &  =  & 
            \displaystyle \sum_{\substack{p \in S \\ c_p < \eta}} (\eta - c_p) + \displaystyle \sum_{\substack{p \in V \\ c_p \geq \eta}} (\eta - c_p) -  \displaystyle \sum_{\substack{p \notin S \\ c_p \geq \eta}} (\eta - c_p) \\
            & = &  \displaystyle \sum_{p=1}^n \amat_{s,p} \fvec_s (1-\fvec_p) +   \displaystyle \sum_{p=1}^n \amat_{p,t} \fvec_p (1-\fvec_t) + \mathcal{C}\\
        \end{array}
    \end{equation*}
    where $\mathcal{C} = \sum_{p \in V; c_p \geq \eta} (\eta - c_p)$
    is a constant, $\fvec_s = 1$, $\fvec_t = 0$, and $\amat$ is
    defined as above.  As $\fvec_s = 1$ and $\fvec_t=0$ enforce that
    $s \in S$ and $t \notin S$, it follows that Eq.~\eqref{eq:pb1} is
    an $s$/$t$ min-cut problem on the transformed graph defined by the
    adjacency matrix $\amat$ over the vertices of $\mathcal{G}$
    augmented by $s$ and $t$.  Note that the above still holds if
    $\wmat$ is a weighted adjacency matrix, and therefore the min-cut
    reformulation can also be applied to a weighted
    network. $\hfill\blacksquare$
\end{proof}

It is therefore possible to use maximal flow algorithms to efficiently
optimize the objective function defined in Equation~\eqref{eq:pb1} and
select a small number of connected SNPs maximally associated with a
phenotype.  In our implementation, we use the Boykov-Kolmogorov
algorithm~\citep{boykov04}.  Although its worst case complexity is in
$\mathcal{O}(n^2n_E n_C)$, where \text{$n_E$} is the number of edges
of the graph and \text{$n_C$} the size of the minimum cut, it performs
much better in practice, particularly when the graph is sparse.  We
refer to this method as \scones, for Selecting CONnected Explanatory
SNPs.

\section{Results}
\label{sec:results}
We evaluate the ability of \scones\ to detect networks of
trait-associated SNPs on simulated datasets and on datasets from an
association mapping study in {\it Arabidopsis thaliana}.

\subsection{Experimental Settings}
\label{sec:exp-settings}
For all of our experiments, we consider the three SNP networks defined
in Section~\ref{sec:intro}: the genomic sequence network (GS), the
gene membership network (GM), and the gene interaction network (GI).
For \scones, the association term $\cvec$ is derived from Linear
SKAT~\citep{wu11}, which makes it possible to correct for covariates
(and therefore population structure). SKAT has been devised to address
rare variants associations problems by grouping SNPs to achieve
statistical significance, but can equally be applied to common
variants.

\paragraph{Univariate linear regression} As a baseline for
comparisons, we run a linear-regression-based single-SNP search for
association, and select those SNPs that are significantly associated
with the phenotype (Bonferroni-corrected $p$-value $\leq 0.05$).

\paragraph{LMM} Similarly, we run a linear mixed model (LMM)
single-SNP search for association~\citep{lippert11}, and select those
SNPs that are significantly associated with the phenotype
(Bonferroni-corrected $p$-value $\leq 0.05$).

\paragraph{Lasso} To compare \scones\ to a method that also considers
all additive effects of SNPs simultaneously with a sparsity
constraint, but without any network regularization, we also run a
Lasso regression~\citep{tibshirani94}, using the SLEP implementation~\citep{slep}
of the Lasso.

\paragraph{ncLasso} In addition, we compare \scones\ to the
network-constrained Lasso ncLasso~\citep{li08}, a version of the Lasso
with sparsity and graph-smoothing constraints equivalent to that of
\scones.  Given a genotype matrix $\gmat$ and a phenotype $\rvec$,
ncLasso solves the following relaxed problem ($\fvec \in
\mathbb{R}^n$):
\begin{equation}
    \argmin_{\fvec \in \mathbb{R}^n} ~\frac{1}{2} \ltwonorm{\gmat
      \fvec - \rvec}^2 + \lambda
    \fvec  ^\top \lmat \fvec + \eta \lonenorm{\fvec}
    \label{eq:nc-lasso}
\end{equation}

The solution for ncLasso proposed by~\citet{li08} requires to compute
and store a single value decomposition of $\lmat$ and is therefore not
applicable when its sizes exceeds $100\,000 \times 100\,000$ by far.
However, a similar solution can be obtained by decomposing $\lmat$ as
the product of the network's incidence matrix with its transpose, an
approach that is much faster (particularly when the network is
sparse).

\paragraph{groupLasso and graphLasso} Eventually, we also compare our
method to the \glasso\ \citep{jacob09}.  The \glasso\ solves the
following relaxed problem:
\begin{equation}
    \argmin_{\fvec \in \mathbb{R}^n} ~\frac{1}{2} \ltwonorm{\gmat
      \fvec - \rvec}^2 + \lambda
    \sum_{g \in \mathcal{G}} \ltwonorm{\fvec^{\mathcal{G}}}
    \label{eq:group-lasso}
\end{equation}
where $\mathcal{G}$ is a set of (possibly overlapping) predefined
groups of SNPs.  We consider the following two
versions:
\begin{itemize}
        \item graphLasso, for which the groups are directly defined
    from the same networks as considered for \scones\ as all pairs of
    vertices connected by an edge;
        \item groupLasso, for which the groups are defined sensibly as
    follows:
    \begin{itemize}
            \item \emph{Genomic sequence groups} (GS): pairs of
        adjacent SNPs (note this gives raise to the same groups as for
        graphLasso with the sequence network);
            \item \emph{Gene membership groups} (GM): SNPs near the
        same gene;
            \item \emph{Gene interaction groups} (GI): SNPs near
        either member of two interacting genes. Here SNPs near genes
        that are not in the interaction network get grouped by gene.
    \end{itemize}
\end{itemize}

We use the SLEP implementation of the non-overlapping group Lasso~\citep{slep},
combined with the trick described by~\citet{jacob09} to compute the
overlapping group Lasso by replicating features in non-overlapping
groups.

\paragraph{Setting the parameters} All methods considered, except for
the univariate linear regression, have parameters (e.g. $\lambda$ and
$\eta$ in the case of \scones) that need to be optimized.  In our
experiments, we run $10$-fold cross-validation grid-search experiments
over ranges of values of the parameters: $7$ values of $\lambda$ and
$\eta$ each for \scones\ and ncLasso, and $7$ values of the parameter
$\lambda$ for the Lasso and the \glasso\ (ranging from $10^{-3}$ to
$10^3$).  We then pick as optimal the parameters leading to the most
stable selection, and report as finally selected the features selected
in all folds.  More specifically, we define stability according to a
consistency index similar to that of~\citet{kuncheva07}.
Following~\cite{kuncheva07}, the consistency index between two feature
sets $S$ and $S'$ is defined relative to the size of their overlap:
$$I_C(S, S') := \frac{\mbox{Observed} (|S \cap S'|) - \mbox{Expected}
  (|S \cap S'|)} 
{\mbox{Maximum} (|S \cap S'|) - \mbox{Expected} (|S \cap S'|)}$$
where
$$\mbox{Maximum} (|S \cap S'|) = min(|S|, |S'|)$$
and $\mbox{Observed} (|S \cap S'|)$ is derived from the hypergeometric
distribution as the expected probability of picking $|S'|$ features
out of $n$ such that $|S \cap S'|$ are among the $|S|$ features in
$S$:

$$ P (|S \cap S'| = r) = \frac{{|S| \choose r} {n-|S| \choose |S'|-r} }{ {n \choose |S'|}}$$

and $$\mbox{Expected}(|S \cap S'|) = \mathbb{E}(P(|S \cap S'| = r)) =
\frac{|S||S'|}{n}.$$

Finally
$$I_C(S, S') = \frac{ n|S \cap S'| - |S||S'| }{ n\min(|S|, |S'|) - |S||S'|}.$$

For an experiment with $k$ folds, the consistency is computed as the
average of the $k(k-1)/2$ pairwise consistencies between sets of
selected features:
$$I_C(S_1, S_2, \dots, S_k) = \frac{k(k-1)}{2} \sum_{i=1}^k \sum_{j=i+1}^k I_C(S_i, S_j).$$

\subsection{Runtime}
We first compare the CPU runtime of \scones\ with that of the linear
regression, ncLasso and graphLasso.  To assess the performance of our
methods, we simulate from $100$ to $200\,000$ SNPs for $200$
individuals and generate exponential random networks with a density of
$2\%$ (chosen as an upper limit on the density of currently available
gene-gene interaction networks) between those SNPs.

We report the real CPU runtime of one cross-validation, for set
parameters, over a single AMD Opteron CPU ($2048$KB, $2600$MHz) with
$512$GB of memory, running Ubuntu 12.04 (Figure
\ref{fig:runtime}). Across a wide range of numbers of SNPs, \scones\
is at least two orders of magnitude faster than graphLasso and one
order of magnitude faster than ncLasso.
\begin{figure*}[t]
    \centering
    \rotatebox{90}{
      \begin{minipage}{\linewidth}
          \includegraphics[width=\linewidth]{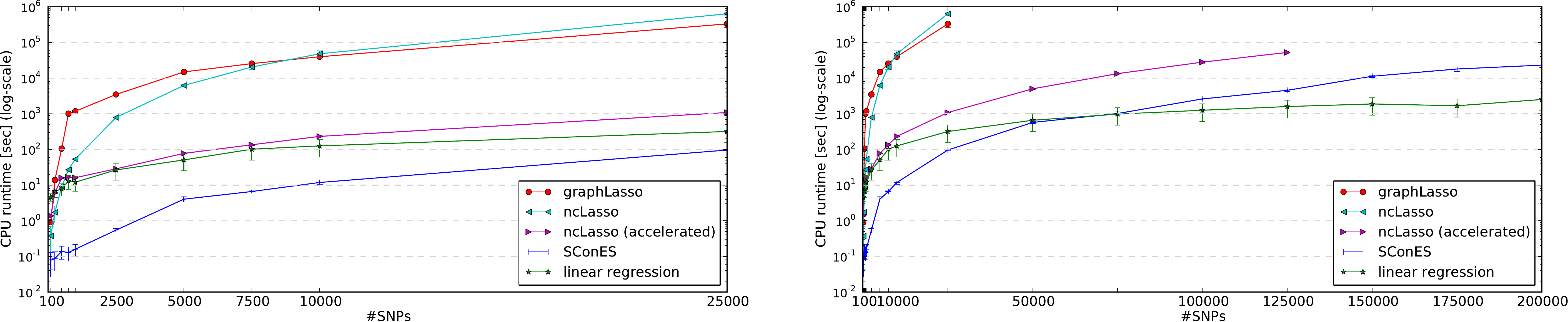}
          \caption{Real CPU runtime comparison between univariate linear
            regression, ncLasso, \glasso\ and \scones, from $100$ to
            $25\,000$ SNPs (left) and from $100$ to $200\,000$ SNPs
            (right). ``ncLasso'' refers to the original implementation
            suggested by~\citet{li08}, and ``ncLasso (accelerated)'' to the
            incidence-matrix based implementation we use here. After three
            weeks, \glasso\ and ncLasso had not finished running for
            $50\,000$ SNPs. The accelerated version of ncLasso ran out of
            memory for $150\,000$ SNPs or more.}
      \end{minipage}
    }
    \label{fig:runtime}
\end{figure*}

\subsection{Simulations}
To assess the performance of our methods, we simulate phenotypes for
$m=500$ real \emph{Arabidopsis thaliana} genotypes ($214\,051$ SNPs),
chosen at random among those made available by~\citet{horton12}, and
the \emph{A. thaliana} protein-protein interaction information from
TAIR~\citep{tpi} (resulting in $55\,584\,646$ SNP-SNP connections). We use a window size
of $20\,000$ base-pairs to define proximity of a SNP to a gene, in
accordance with the threshold used for the interpretation of GWAS
results in~\citet{atwell10}. Restricting ourselves to $1,000$ randomly
picked SNPs with minor allele frequency larger than $10\%$, we pick
$20$ of the SNPs to be causal, and generate phenotypes $y_i = w ^\top
g_i + \epsilon$, where both the support weights $w$ and the noise
$\epsilon$ are normally distributed.  We consider the following
scenarios: (a) the causal SNPs are randomly distributed in the
network; (b) the causal SNPs are adjacent on the genomic sequence; (c)
the causal SNPs are near the same gene; (d-f) the causal SNPs are near
either of two, three, and five interacting genes, respectively.  We
then select SNPs using univariate linear regression, Lasso, ncLasso,
the two flavors of \glasso, and \scones\ as described in
Section~\ref{sec:exp-settings}.  We repeat each experiment $30$ times,
and compare the selected SNPs of either approach with the true causal
ones in terms of power (fraction of causal SNPs selected) and false
discovery rate (FDR, fraction of selected SNPs that are not causal).
We summarize the results with F-scores (harmonic mean of power and one
minus FDR) in Table~\ref{tab:fscore-simu}.

As \scones\ returns a binary feature selection rather than a feature
ranking, it is not possible to draw FDR curves or compare powers at
same FDR as is often done when evaluating such methods.
Figure~\ref{fig:precrec-simu} presents the average FDR and power of
the different algorithms under three of the scenarios, depending on
the network used.  The closer the FDR-power point representing an
algorithm to the upper-left corner, the better this algorithm at
maximizing power while minimizing FDR.  As it is easy to get better
power by selecting more SNPs, we also report on the same figure the
number of SNPs selected by each algorithm, and show that it remains
reasonably close to the true value of $20$ causal SNPs.

\scones\ is systematically better than its state-of-the-art comparison
partners at leveraging structural information to retrieve the
connected SNPs that were causal.  Only when the groups perfectly match
the causal structure (Scenario (d)) can groupLasso outperform SConES.
While the performance of \scones\ and ncLasso does depend on the
network, the \glasso\ is much more sensitive to the definition of its
groups.  Furthermore, we observe that removing a small fraction ($1\%$
to $15\%$) of the edges between causal features does not harm the
performance of \scones\ (see Table~\ref{tab:remove-edges}). This means
that \scones\ is robust to missing edges, an important point when the
biological network used is likely to be incomplete.  Nevertheless, the
performance of \scones, as that of all other network-regularized
approaches, is strongly negatively affected when the network is
entirely inappropriate (Scenario (a)).  In addition, the decrease in
performance from Scenario (c) to Scenario (f), when the number of
interacting genes near which the causal SNPs are located increases
from $1$ to $5$, indicates that \scones, like its
structure-regularized comparison partners, performs better when the
causal SNPs are less spread out in the network.  Finally, ncLasso is
both slower and less performant than \scones. This indicates that
solving the feature selection problem we pose directly, rather than
its relaxed version, allows for better recovery of true causal
features.

\begin{figure}[t]
    \centerline{\subfloat[Scenario (b): The true causal SNPs belong to
      the same genomic segment] {
        \includegraphics[width=\linewidth]{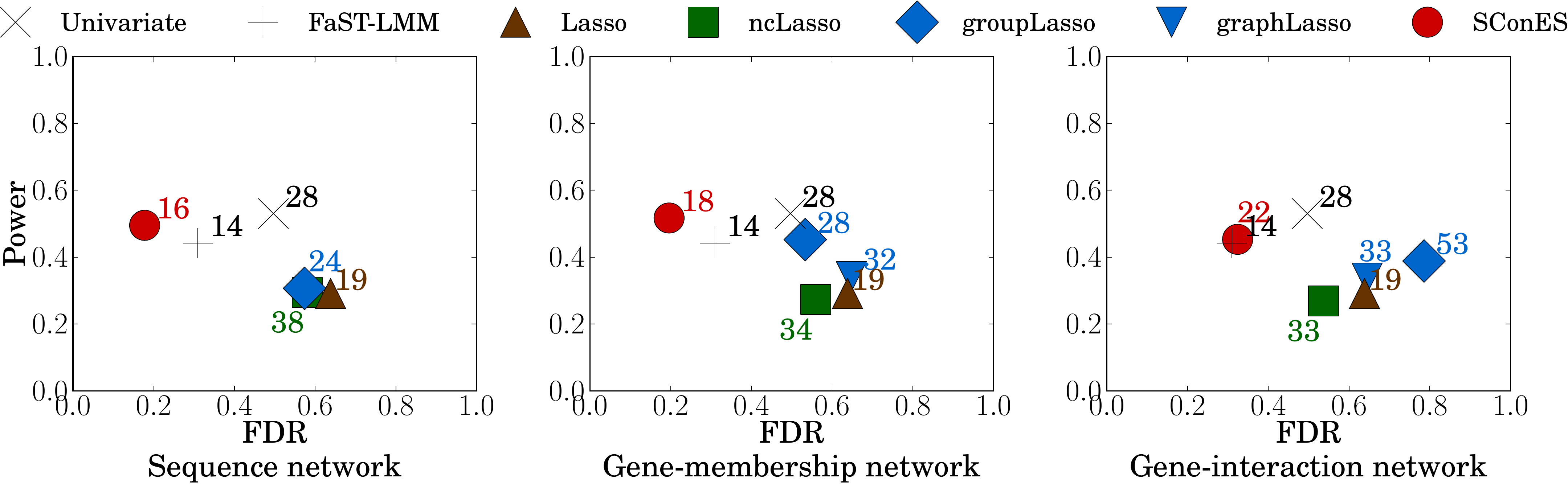}
        \label{fig:precrec-simu-line}
      }} \vspace{-10pt} \centerline{\subfloat[Scenario (c): The true
      causal SNPs are near the same gene] {
        \includegraphics[width=\linewidth]{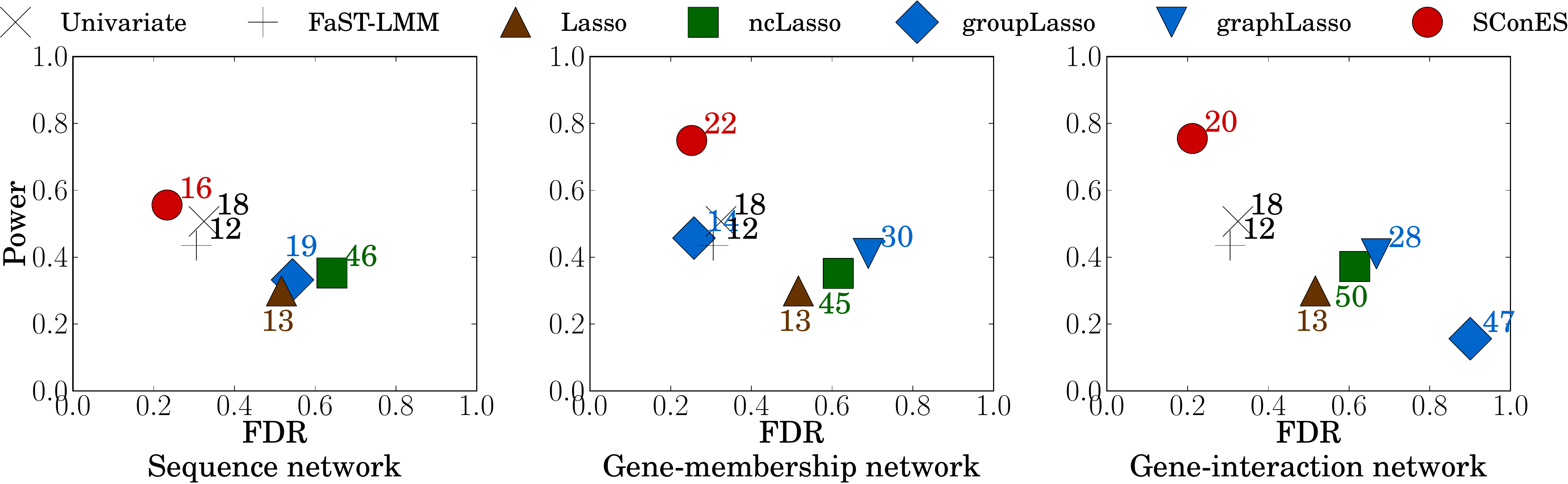}
        \label{fig:precrec-simu-gene}
      }} \vspace{-10pt} \centerline{\subfloat[Scenario (f): The true
      causal SNPs are near any of five interacting genes] {
        \includegraphics[width=\linewidth]{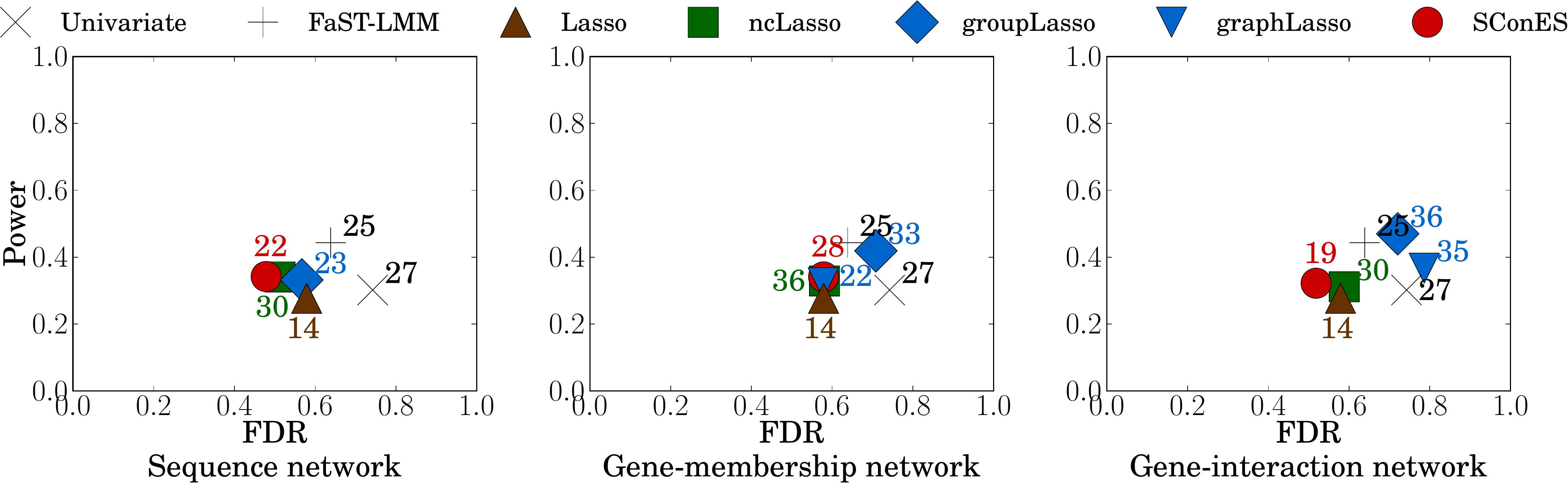}
        \label{fig:precrec-simu-pathway}
      }}
    \caption{Power and false discovery rate (FDR) of SConES, compared
      to state-of-the-art Lasso algorithms and a baseline univariate
      linear regression, in three different data simulation
      scenarios. Best methods are closest to the upper-left
      corner. Numbers denote the number of SNPs selected by the
      method.}
    \label{fig:precrec-simu}
    \vspace{-15pt}
\end{figure}

\renewcommand{\tabcolsep}{1pt}
\begin{table}[t]
    \centering
    \begin{tabular}[h]{|ll|c|c|c|} \hline
        & & (a) & (b) & (c) \\ \hline 
        \multicolumn{2}{|l|}{Univariate} & $0.26 \pm 0.07$  &
        $0.29 \pm 0.12$  & $0.28 \pm 0.14$ \\ \hline
        \multicolumn{2}{|l|}{LMM} & ${\it 0.32 \pm
          0.01}$  & $0.35 \pm 0.01$  & $0.33 \pm
        0.01$ \\ \hline
        \multicolumn{2}{|l|}{Lasso} & $\bm{0.35 \pm 0.01}$  & $0.32 \pm 0.02$  & $0.36 \pm 0.01$  \\ \hline 
        \multirow{3}{*}{ncLasso} & GS & $0.17 \pm 0.01$  & $0.25 \pm 0.02$  & $0.25 \pm 0.01$  \\ 
        & GM & $0.17 \pm 0.01$  & $0.26 \pm 0.02$  & $0.26 \pm 0.02$ \\ 
        & GI & $0.19 \pm 0.01$  & $0.26 \pm 0.02$  & $0.26 \pm 0.02$ \\ \hline 
        \multirow{3}{*}{groupLasso} & GS & $0.23 \pm 0.01$  & $0.30 \pm 0.01$  & $0.34 \pm 0.01$ \\ 
        & GM & $0.12 \pm 0.00$  & $0.44 \pm 0.02$  & $0.55 \pm 0.01$ \\ 
        & GI & $0.09 \pm 0.00$  & $0.26 \pm 0.02$  & $0.11 \pm 0.01$ \\ \hline 
        \multirow{3}{*}{graphLasso} & GS & $0.23 \pm 0.01$ & $0.30 \pm 0.01$  & $0.34 \pm 0.01$ \\ 
        & GM & $0.23 \pm 0.01$ & $0.28 \pm 0.01$  & $0.33 \pm 0.01$ \\ 
        & GI & $0.22 \pm 0.01$ & $0.28 \pm 0.01$  & $0.34 \pm 0.01$ \\ \hline 
        \multirow{3}{*}{SConES} & GS & $0.21 \pm 0.01$  & ${\it 0.55 \pm 0.04}$  & $0.57 \pm 0.04$ \\ 
        & GM & $0.19 \pm 0.02$  & $\bm{0.58 \pm 0.03}$  & ${\it 0.75 \pm 0.03}$ \\ 
        & GI & $0.20 \pm 0.02$  & $0.48 \pm 0.03$  & $\bm{0.78 \pm 0.03}$ \\ \hline
        & & (d) & (e) & (f) \\ \hline 
        \multicolumn{2}{|l|}{Univariate} & $0.27 \pm 0.07$  & $0.26
        \pm 0.07$  & $0.23 \pm 0.08$ \\ \hline
        \multicolumn{2}{|l|}{LMM}  & $0.36 \pm 0.02$  & $0.38 \pm 0.01$  & ${\it 0.33 \pm 0.01}$ \\ \hline          
        \multicolumn{2}{|l|}{Lasso} & $0.36 \pm 0.01$  & $0.37 \pm 0.01$  & $0.32 \pm 0.01$ \\ \hline 
        \multirow{3}{*}{ncLasso} & GS & $0.45 \pm 0.01$  & $0.38 \pm 0.02$  & $0.30 \pm 0.01$ \\ 
        & GM & $0.38 \pm 0.01$  & $0.29 \pm 0.01$  & $0.27 \pm 0.01$ \\ 
        & GI & $0.43 \pm 0.02$  & $0.34 \pm 0.02$  & $0.28 \pm 0.01$ \\ \hline 
        \multirow{3}{*}{groupLasso} & GS & $0.37 \pm 0.01$  & $0.36 \pm 0.02$  & $0.32 \pm 0.01$ \\ 
        & GM & ${\it 0.50 \pm 0.01}$  & ${\it 0.40 \pm 0.01}$  & ${\it
          0.33 \pm 0.01}$ \\ 
        & GI & $\bm{0.54 \pm 0.01}$  & ${\it 0.40 \pm 0.01}$  & $\bm{0.34 \pm 0.01}$ \\ \hline 
        \multirow{3}{*}{graphLasso} & GS & $0.37 \pm 0.01$  & $0.36 \pm 0.02$  & $0.32 \pm 0.01$ \\ 
        & GM & $0.36 \pm 0.01$  & $0.31 \pm 0.01$  & $0.31 \pm 0.01$ \\ 
        & GI & $0.33 \pm 0.01$  & $0.30 \pm 0.01$  & $0.27 \pm 0.01$ \\ \hline 
        \multirow{3}{*}{SConES} & GS & ${\it 0.50 \pm 0.01}$  &
        $\bm{0.43 \pm 0.02}$  & ${\it 0.33 \pm 0.02}$ \\ 
        & GM & $0.49 \pm 0.01$  & ${\it 0.40 \pm 0.02}$  & $0.32 \pm 0.02$ \\ 
        & GI & $0.49 \pm 0.01$  & $0.39 \pm 0.01$  & $\bm{0.34 \pm 0.02}$ \\ \hline
    \end{tabular}
    \caption{F-scores of SConES, compared to state-of-the-art Lasso
      algorithms and a baseline univariate linear regression, in six different
      data simulation scenarios: The true causal SNPs are 
      (a) unconnected;
      (b) adjacent on the genomic sequence;
      (c) near the same gene;
      (d) near either of the same $2$ connected genes;
      (e) near either of the same $3$ connected genes;
      (f) near either of the same $5$ connected genes.
      Best performance in bold and second best in italics.
      ``GS'': Genomic sequence network.
      ``GM'': Gene membership network.  ``GI'': Gene interaction network.}
    \label{tab:fscore-simu}
\end{table}

\begin{table}[ht]
    \centering
    \begin{tabular}[h]{|r|r|r|r|r|r|}
        \hline
        \multirow{2}{*}{Scenario} & \multicolumn{5}{c|}{Fraction of edges removed} \\
        \cline{2-6}
        & $0\%$ & $2\%$ & $5\%$ & $10\%$ & $15\%$ \\ \hline
        (b) & $0.58 \pm 0.03$ & $0.58 \pm 0.03$ & $0.58 \pm 0.03$ &
        $0.57 \pm 0.03$ & $0.55 \pm 0.03$ \\
        (c) & $0.75 \pm 0.03$ & $0.75 \pm 0.03$ & $0.75 \pm 0.03$ &
        $0.75 \pm 0.03$ & $0.62 \pm 0.03$ \\
        (f) & $0.34 \pm 0.02$ & $0.34 \pm 0.02$ & $0.34 \pm 0.02$ &
        $0.33 \pm 0.02$ & $0.29 \pm 0.02$ \\
        \hline
    \end{tabular}
    \caption{Effect on the F-scores of SConES of removing a small
      fraction of the network edges.
      Results reported for SConES+GM in three different scenarios:
      The true causal SNPs are 
      (b) adjacent on the genomic sequence;
      (c) near the same gene;
      (f) near either of the same $5$ connected genes.}
    \label{tab:remove-edges}
\end{table}

\subsection{{\it Arabidopsis} Flowering Time Phenotypes}
We then apply our method to a large collection of 17 \emph{A.
  thaliana} flowering times phenotypes from~\citet{atwell10} (up to
$194$ individuals, $214\,051$ SNPs).  The groups and networks are
again derived from the TAIR protein-protein interaction data~\citep{tpi}.  We
filter out SNPs with a minor allele frequency lower than $10\%$, as is
typical in {\it A. thaliana} GWAS studies.  We use the first principal
components of the genotypic data as covariates to correct for
population structure~\citep{price06}: the number of principal
components is chosen by adding them one by one until the genomic
control is close to $1$ (see Figure~\ref{fig:genomic}).

\begin{figure}[h] 
   \centerline{\includegraphics[width=\linewidth]{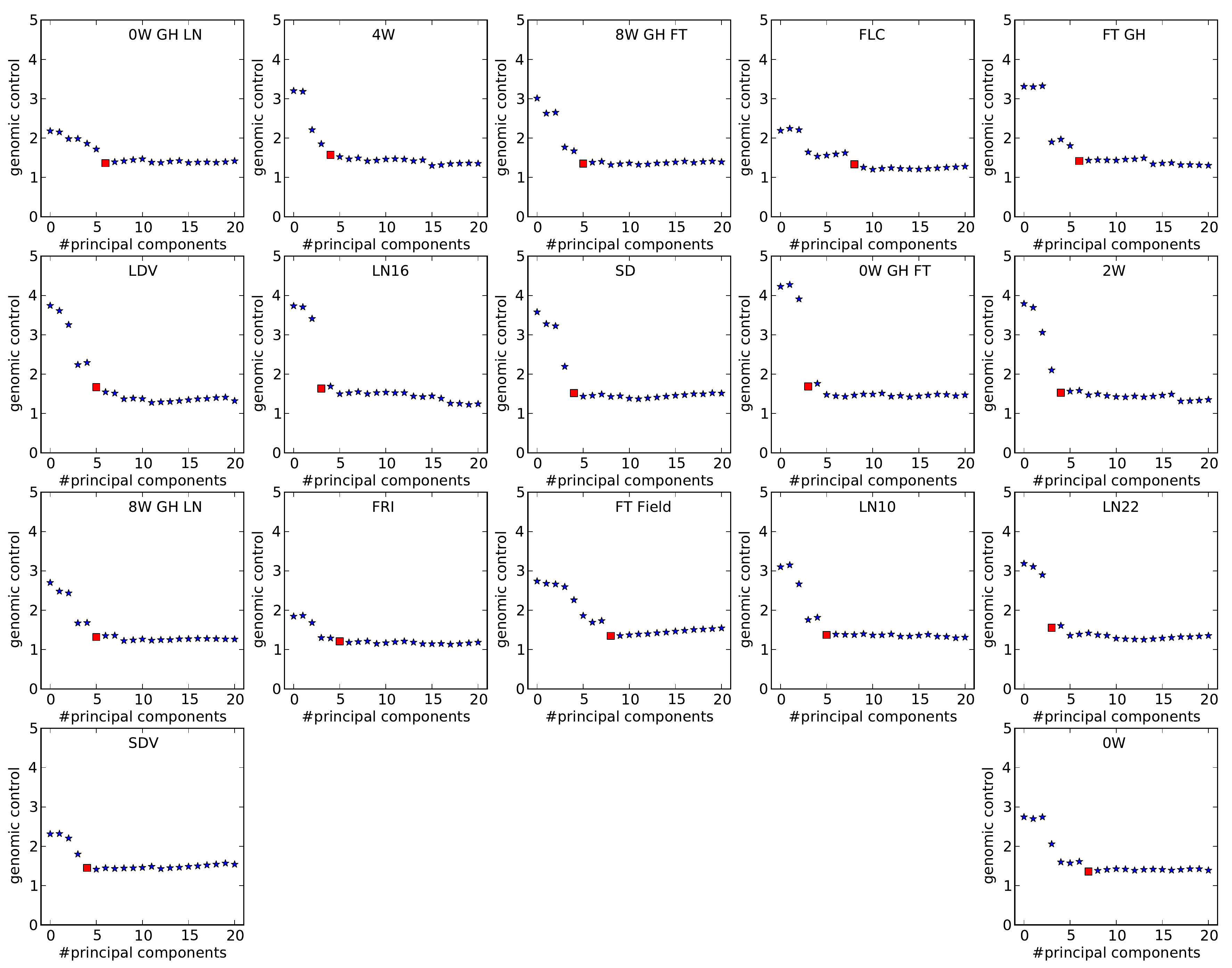}}
      \caption{Genomic control values ($\lambda$) per phenotype for an increasing number of principal components (PC). The red rectangle indicates the number of PCs selected for this particular phenotype.}
   \label{fig:genomic}
\end{figure}

The direct competitors of \scones\ on this problem are the methods
that also impose graph constraints on the SNPs they select, namely
graphLasso and ncLasso.  However, graphLasso does not scale to
datasets such as ours with more than 200k SNPs (see
Figure~\ref{fig:runtime}). Hence we had to exclude it from our
experiments. Note that while even our accelerated implementation of
ncLasso could not be run on more than $125\,000$ SNPs in our
simulations, the networks derived for {\it A. thaliana} are sparser
than that used in the simulations, which makes it possible to run
ncLasso on this data.

Instead, we compare \scones\ to ncLasso and groupLasso, which uses
pairs of neighboring SNPs, SNPs from the same gene or SNPs from
interacting genes as pre-defined groups. Note that groupLasso on
sequence-neighboring SNPs is identical to graphLasso on the sequence
network, which is the only instance of graphLasso whose computation is
practically feasible on this dataset.  We run Lasso, ncLasso,
groupLasso and \scones\ on the flowering time phenotypes as described
in Section~\ref{sec:exp-settings}. However, for many of the
phenotypes, the Lasso approaches select large number of SNPs (more
than $10\,000$), which makes the results hard to interprete. Using
cross-validated predictivity, as is generally done for Lasso, still
does not entirely solve this issue, particularly for large group sizes
(see Tables~\ref{tab:ft-res-nosizefilter}
and~\ref{tab:ft-res-predictivity}). We therefore filter out solutions
containing more than $1\%$ of the total number of SNPs before using
consistency to select the optimal parameters.

To evaluate the quality of the SNPs selected, we perform ridge
regression on each phenotype in a cross-validation scheme that uses
only the selected SNPs and report its average Pearson's squared
correlation coefficient in Figure~\ref{fig:ft-pearson}.  We also
report, as an additional baseline, the cross-validated predictivity of
a standard best linear unbiased prediction (BLUP)
~\citep{henderson75}.  While the features selected by groupLasso+GS
achieve higher predictivity than \scones+GS on most phenotypes, the
features selected by \scones+GM are at least as predictive as those
selected by groupLasso+GM in two thirds of the phenotypes; the picture
is the same for \scones+GI, whose selected SNPs are on average more
predictive than those of groupLasso+GI. The superiority of groupLasso
in that respect is to be expected, as predicitivity is directly
optimized by the regression.  Also note that in $80\%$ of the cases,
if any of the feature selection methods achieves high predictivity
($R^2>0.6$), \scones\ outperforms all other methods including BLUP.

Next, we checked whether the selected SNPs from the three methods
coincide with flowering time genes from the literature.  We report in
Table~\ref{tab:ft-res} the number of SNPs selected by each of the
methods and the proportion of these SNPs that are near flowering time
candidate genes listed by~\citet{segura12}.  Here, the picture is
reversed: \scones+GS and groupLasso+GI retrieve the highest ratio of
SNPs near candidate genes, while groupLasso+GS, \scones+GI and
\scones+GM show lower ratios. At first sight, it seems surprising that
the methods with highest predictive power retrieve the least SNPs near
candidate genes.

To further investigate this phenomenon, we record how many distinct
flowering time candidate genes are retrieved on average by the various
methods.  A gene is considered retrieved if the method selects a SNP
near it.  Our results are shown in Table~\ref{tab:ft-summary}. Methods
retrieving a large fraction of SNPs near candidate genes do not
necessarily retrieve the largest number of distinct candidate
genes. Good predictive power, as shown in Figure~\ref{fig:ft-pearson},
however, seems to correlate with the number of distinct candidate
genes selected by an algorithm, not with the percentage of selected
SNPs near candidate genes.  groupLasso+GI has the highest fraction of
candidate gene SNPs among all methods, but detects only three distinct
candidate genes. This is probably due to groupLasso selecting entire
genes or gene pairs; if groupLasso detects a candidate gene, it will
pick most of the SNPs near that gene, which leads to its high
candidate SNP ratio in Table~\ref{tab:ft-res}.

We also compare the selected SNPs to those deemed significant by a
linear mixed model ran on the full data (see
Table~\ref{tab:ft-emma-hits}). \scones\ systematically recovers more
of those SNPs than the Lasso approaches.

To summarize, \scones\ is able to select SNPs that are highly
predictive of the phenotype. Among all methods, \scones+GM discovers
the largest number of distinct genes whose involvement in flowering
time is supported by the literature.

\begin{table*}[t]
    \centering
    \begin{tabular}{|r|r|r|r|r|r|r|r|r|r|r|r|r|} \hline
        \multirow{2}{*}{Phenotype} & \multirow{2}{*}{Univariate} &
        \multirow{2}{*}{LMM} & 
        \multirow{2}{*}{Lasso} & \multicolumn{3}{c|}{groupLasso} &
        \multicolumn{3}{c|}{ncLasso} &
        \multicolumn{3}{c|}{SConES} \\
        \cline{5-13} & & & & GS & GM & GI & GS & GM & GI & GS & GM &
        GI \\ \hline
        0W  & $0/3$ & $0/0$ & $1/29$ & $33/288$ & $59/706$ & $144/547$ & $40/1077$ & $14/318$ & $14/318$ & $\bm{123/271}$ & $0/85$ & $0/69$\\
        0W GH LN  & $0/0$ & $0/0$ & $2/20$ & $13/205$ & $54/478$ & $\bm{128/321}$ & $31/981$ & $11/320$ & $11/320$ & $92/1251$ & $92/1252$ & $92/1253$\\
        4W  & $1/8$ & $1/2$ & $15/129$ & $7/52$ & $48/1489$ & $\bm{80/436}$ & $2/238$ & $6/298$ & $6/298$ & $104/1670$ & $66/1078$ & $42/859$\\
        8W GH FT  & $0/5$ & $0/1$ & $10/143$ & $\bm{5/16}$ & $66/1470$ & $0/0$ & $14/427$ & $11/398$ & $11/398$ & $26/322$ & $26/322$ & $26/319$\\
        FLC  & $0/1$ & $0/1$ & $1/31$ & $2/95$ & $0/101$ & $0/214$ & $4/135$ & $1/35$ & $1/35$ & $\bm{115/1592}$ & $0/2$ & $0/2$\\
        FT GH  & $0/1$ & $2/10$ & $\bm{7/46}$ & $8/106$ & $90/841$ & $177/1417$ & $37/1434$ & $42/1709$ & $42/1709$ & $0/626$ & $0/59$ & $0/59$\\
        LDV  & $0/4$ & $\bm{1/2}$ & $10/80$ & $8/32$ & $0/0$ & $0/0$ & $14/437$ & $7/177$ & $7/177$ & $39/674$ & $86/1381$ & $54/1091$\\
        LN16  & $0/5$ & $0/0$ & $9/222$ & $0/95$ & $138/957$ & $89/1307$ & $22/1094$ & $33/1323$ & $33/1323$ & $\bm{73/73}$ & $0/3$ & $0/4$\\
        SD  & $0/2$ & $0/1$ & $3/36$ & $36/569$ & $51/863$ & $84/721$ & $20/466$ & $10/224$ & $10/224$ & $\bm{7/59}$ & $\bm{7/59}$ & $\bm{7/59}$\\
        0W GH FT  & $0/9$ & $\bm{1/3}$ & $20/194$ & $49/654$ & $52/898$ & $241/1258$ & $63/1597$ & $84/1997$ & $84/1997$ & $0/6$ & $29/317$ & $29/317$\\
        2W  & $0/12$ & $0/6$ & $4/36$ & $7/79$ & $93/610$ & $\bm{126/810}$ & $28/1006$ & $43/1256$ & $43/1256$ & $76/756$ & $78/1185$ & $25/892$\\
        8W GH LN  & $0/2$ & $0/3$ & $8/122$ & $13/168$ & $0/0$ & $0/0$ & $19/493$ & $21/501$ & $21/501$ & $\bm{11/73}$ & $75/776$ & $68/757$\\
        FRI  & $6/11$ & $5/9$ & $6/18$ & $8/64$ & $8/20$ & $\bm{10/10}$ & $2/9$ & $2/4$ & $2/4$ & $101/1266$ & $101/1271$ & $101/1274$\\
        FT Field  & $2/4$ & $0/0$ & $1/79$ & $5/37$ & $51/221$ & $\bm{52/72}$ & $18/709$ & $5/238$ & $5/238$ & $4/8$ & $4/8$ & $4/8$\\
        LN10  & $0/1$ & $0/0$ & $0/12$ & $2/34$ & $\bm{18/121}$ & $0/202$ & $12/644$ & $12/649$ & $12/649$ & $165/1921$ & $0/91$ & $0/91$\\
        LN22  & $2/14$ & $0/0$ & $6/65$ & $0/12$ & $33/894$ & $81/1023$ & $23/501$ & $26/506$ & $26/506$ & $\bm{140/1378}$ & $\bm{140/1378}$ & $140/1378$\\
        SDV & $0/5$ & $0/1$ & $4/208$ & $3/94$ & $1/721$ &
        $105/936$ & $14/379$ & $15/384$ & $15/384$ & $\bm{53/454}$ & $0/8$ & $0/8$\\
        \hline
    \end{tabular}
    \caption{Associations detected close to known candidate genes, for
      all flowering time phenotypes of \emph{Arabidopsis thaliana}. We
      report the number of selected SNPs near candidate genes,
      followed by the total number of selected SNPs. Largest ratio in
      bold. ``GS'': Genomic sequence network.  ``GM'': Gene membership
      network.  ``GI'': Gene interaction network. }
    \label{tab:ft-res}
    \vspace{-10pt}
\end{table*}

\begin{table}[t]
    \centering
    \begin{tabular}{|r|r|r|r|} \hline & \#SNPs & near candidate
        genes & candidate genes hit \\\hline
        Univariate & $5$ & $0.09$ & $0.35$ \\
        LMM & $2$ & $0.12$ & $0.35$ \\
        Lasso & $86$ & $0.09$ & $3.82$ \\
        groupLasso GS & $153$ & $0.10$ & $4.35$ \\
        groupLasso GM & $611$ & $0.09$ & $1.35$ \\
        groupLasso GI & $546$ & $0.20$ & $2.65$ \\
        ncLasso GS & $684$ & $0.04$ & $4.88$ \\
        ncLasso GM & $608$ & $0.06$ & $4.59$ \\
        ncLasso GI & $608$ & $0.06$ & $4.59$ \\
        SConES GS & $729$ & $0.18$ & $11.53$ \\
        SConES GM & $546$ & $0.08$ & $14.82$ \\
        SConES GI & $496$ & $0.07$ & $12.24$ \\
        \hline
    \end{tabular}    
    \caption{Summary statistics, averaged over the \textit{Arabidopsis
        thaliana} flowering time phenotypes: average total number of
      selected SNPs (``\#SNPs''), average proportion of selected SNPs
      near candidate genes (``near candidate genes'') and average
      number of different candidate genes recovered (``candidate genes
      hit'') .  ``GS'': Genomic sequence network.  ``GM'': Gene
      membership network.  ``GI'': Gene interaction network.}
    \label{tab:ft-summary}
    \vspace{-10pt}
\end{table}

\begin{table}
    \centering
    \begin{tabular}[h]{|l|c|c|ccc|ccc|} \hline
        \multirow{2}{*}{Phenotype} & \multirow{2}{*}{LinReg} & \multirow{2}{*}{Lasso} & \multicolumn{3}{c|}{groupLasso} & \multicolumn{3}{c|}{SConES} \\
        & & & GS & GM & GI & GS & GM & GI \\\hline 
        0W  & $3$ & $53545$ & $53567$ & $53572$ & $53816$ & $271$ & $85$ & $69$ \\
        0W GH LN  & $0$ & $53759$ & $53633$ & $53706$ & $53785$ & $1251$ & $1252$ & $1253$ \\
        4W  & $8$ & $60243$ & $60276$ & $60354$ & $60572$ & $1670$ & $1078$ & $859$ \\
        8W GH FT  & $5$ & $59759$ & $60426$ & $59972$ & $59443$ & $322$ & $322$ & $319$ \\
        FLC  & $1$ & $52587$ & $52491$ & $52588$ & $8672$ & $4624$ & $3115$ & $2$ \\
        FT GH  & $1$ & $53824$ & $53802$ & $54024$ & $54210$ & $626$ & $59$ & $59$ \\
        LDV  & $4$ & $53670$ & $53797$ & $53600$ & $7870$ & $674$ & $1381$ & $1091$ \\
        LN16  & $5$ & $55308$ & $55388$ & $55486$ & $2914$ & $2113$ & $3785$ & $4$ \\
        SD  & $2$ & $54284$ & $54418$ & $54481$ & $12664$ & $59$ & $59$ & $59$ \\
        0W GH FT  & $9$ & $57189$ & $57212$ & $57337$ & $13199$ & $6$ & $317$ & $317$ \\
        2W  & $12$ & $55967$ & $56062$ & $56229$ & $56308$ & $756$ & $1185$ & $892$ \\
        8W GH LN  & $2$ & $60549$ & $60520$ & $60587$ & $2153$ & $73$ & $776$ & $757$ \\
        FRI  & $11$ & $50556$ & $50631$ & $50827$ & $5731$ & $1266$ & $1271$ & $1274$ \\
        FT Field  & $4$ & $46683$ & $47122$ & $221$ & $47322$ & $28805$ & $8$ & $8$ \\
        LN10  & $1$ & $54794$ & $54967$ & $54848$ & $2858$ & $1921$ & $91$ & $91$ \\
        LN22  & $14$ & $56537$ & $56585$ & $56601$ & $15486$ & $1378$ & $1378$ & $1378$ \\
        SDV  & $5$ & $59242$ & $59368$ & $59355$ & $59110$ & $454$ & $8$ & $8$ \\
        \hline
    \end{tabular}
    \caption{Number of SNPs selected, for all flowering time phenotypes of \emph{Arabidopsis
        thaliana}, when using \textbf{consistency without a cardinality constraint}  to select parameters. 
      ``GS'': Genomic sequence network.
      ``GM'': Gene membership network.  ``GI'': Gene interaction network.}
    \label{tab:ft-res-nosizefilter}
\end{table}

\begin{table}
    \centering
    \begin{tabular}[h]{|l|c|c|ccc|ccc|} \hline
        \multirow{2}{*}{Phenotype} & \multirow{2}{*}{LinReg} & \multirow{2}{*}{Lasso} & \multicolumn{3}{c|}{groupLasso} & \multicolumn{3}{c|}{SConES} \\
        & & & GS & GM & GI & GS & GM & GI \\\hline 
        0W  & $3$ & $1077$ & $3289$ & $1827$ & $17854$ & $5037$ & $5002$ & $4968$ \\
        0W GH LN  & $0$ & $765$ & $205$ & $1796$ & $19306$ & $1234$ & $1253$ & $1253$ \\
        4W  & $8$ & $994$ & $2902$ & $8571$ & $21019$ & $1670$ & $1910$ & $1890$ \\
        8W GH FT  & $5$ & $1614$ & $2932$ & $1470$ & $20347$ & $322$ & $322$ & $281$ \\
        FLC  & $1$ & $1158$ & $3517$ & $9805$ & $19905$ & $5126$ & $5126$ & $5119$ \\
        FT GH  & $1$ & $197$ & $591$ & $2081$ & $17562$ & $6163$ & $6161$ & $6163$ \\
        LDV  & $4$ & $676$ & $1832$ & $1809$ & $18329$ & $2867$ & $2867$ & $2867$ \\
        LN16  & $5$ & $2319$ & $5696$ & $15027$ & $18252$ & $5824$ & $6326$ & $6243$ \\
        SD  & $2$ & $1920$ & $569$ & $15933$ & $18655$ & $59$ & $59$ & $59$ \\
        0W GH FT  & $9$ & $194$ & $654$ & $2514$ & $20246$ & $312$ & $316$ & $314$  \\
        2W  & $12$ & $135$ & $387$ & $11368$ & $18886$ & $3186$ & $3001$ & $3210$ \\
        8W GH LN  & $2$ & $1654$ & $3031$ & $60602$ & $20109$ & $816$ & $827$ & $826$ \\
        FRI  & $11$ & $1013$ & $3335$ & $9422$ & $19281$ & $1274$ & $1274$ & $1273$ \\
        FT Field  & $4$ & $1029$ & $2297$ & $47005$ & $2242$ & $8$ & $7$ & $7$  \\
        LN10  & $1$ & $607$ & $184$ & $1871$ & $18674$ & $7840$ & $7846$ & $7846$ \\
        LN22  & $14$ & $393$ & $1132$ & $4019$ & $21308$ & $1377$ & $1321$ & $1378$ \\
        SDV  & $5$ & $1860$ & $10073$ & $16599$ & $19651$ & $8757$ & $4362$ & $8812$ \\
        \hline
    \end{tabular}
    \caption{Number of SNPs selected, for all flowering time phenotypes of \emph{Arabidopsis
        thaliana}, when using the \textbf{predictivity} (cross-validated $R^2$ between actual phenotype and 
      that predicted by a ridge-regression trained only on the selected SNPs) to select parameters. 
      ``GS'': Genomic sequence network.
      ``GM'': Gene membership network.  ``GI'': Gene interaction network.}
    \label{tab:ft-res-predictivity}
\end{table}

\begin{table}
     \centering
      \begin{tabular}[h]{|l|c|c|c|c|ccc|ccc|} \hline
          \multirow{2}{*}{Phenotype} &  \multirow{2}{*}{EMMA} &  \multirow{2}{*}{Univariate} &  \multirow{2}{*}{FaST-LMM} &  \multirow{2}{*}{Lasso} & \multicolumn{3}{c|}{groupLasso} & \multicolumn{3}{c|}{SConES} \\ \cline{6-11}
          & & & & & GS & GM & GI & GS & GM & GI \\\hline
          4W	 & $2$ &  $50\%$ & $100\%$ & $0\%$ & $0\%$ & $0\%$ & $0\%$ & $50\%$ & $50\%$ & $50\%$  \\
          8W GH FT	 & $1$ &  $0\%$ & $100\%$ & $0\%$ & $0\%$ & $0\%$ & $0\%$ & $0\%$ & $0\%$ & $0\%$  \\
          FLC	 & $1$ &  $0\%$ & $100\%$ & $0\%$ & $0\%$ & $0\%$ & $0\%$ & $0\%$ & $0\%$ & $0\%$  \\
          FT GH	 & $10$ &  $0\%$ & $100\%$ & $0\%$ & $0\%$ & $0\%$ & $0\%$ & $0\%$ & $0\%$ & $0\%$  \\
          LDV	 & $2$ &  $0\%$ & $100\%$ & $0\%$ & $0\%$ & $0\%$ & $0\%$ & $0\%$ & $0\%$ & $0\%$  \\
          SD	 & $1$ &  $100\%$ & $100\%$ & $0\%$ & $0\%$ & $0\%$ & $0\%$ & $100\%$ & $100\%$ & $100\%$  \\
          0W GH FT	 & $3$ &  $67\%$ & $100\%$ & $0\%$ & $0\%$ & $0\%$ & $0\%$ & $0\%$ & $0\%$ & $0\%$  \\
          2W	 & $6$ &  $33\%$ & $100\%$ & $0\%$ & $0\%$ & $0\%$ & $0\%$ & $17\%$ & $33\%$ & $17\%$  \\
          8W GH LN	 & $3$ &  $33\%$ & $100\%$ & $0\%$ & $0\%$ & $0\%$ & $0\%$ & $0\%$ & $33\%$ & $33\%$  \\
          FRI	 & $9$ &  $100\%$ & $100\%$ & $100\%$ & $89\%$ & $56\%$ & $0\%$ & $100\%$ & $100\%$ & $100\%$  \\
          SDV	 & $1$ &  $100\%$ & $100\%$ & $0\%$ & $0\%$ & $0\%$ & $0\%$ & $0\%$ & $0\%$ & $0\%$  \\ \hline
      \end{tabular}
      \caption{Fraction of SNPs deemed significantly associated with the
          phenotype by EMMA run the full dataset (number of such SNPs
          reported in the second column) that
          were selected. We only report the phenotypes for which EMMA returned
          at least one significant SNP.}
      \label{tab:ft-emma-hits}    
\end{table}

\begin{figure*}[t]
    \vspace{-20pt} \centerline{\subfloat[Genomic sequence networks] {
        \includegraphics[width=0.6\textwidth]{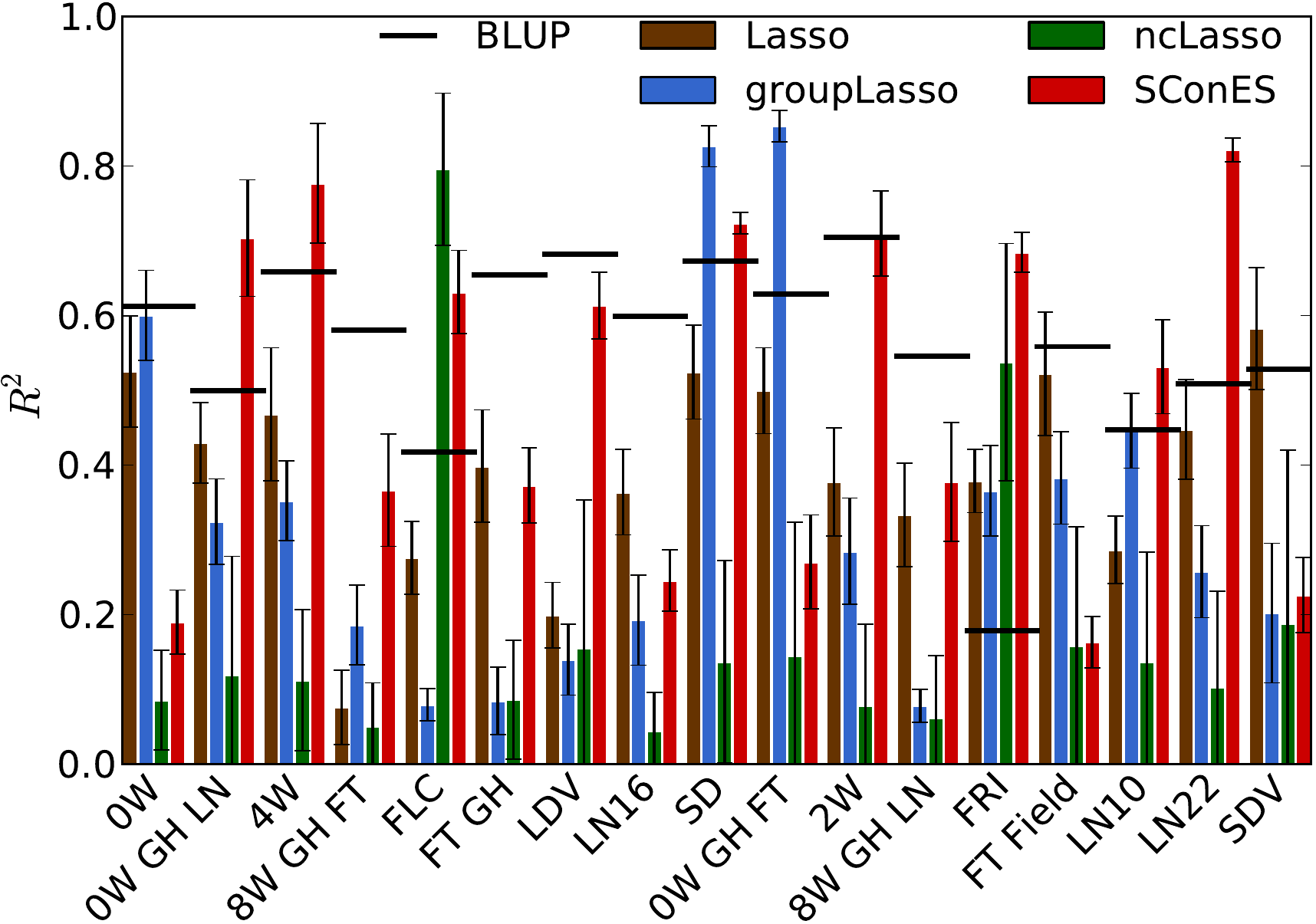}
        \label{fig:ft-pearson-seq}
      }}
    \centerline{\subfloat[Gene membership networks] {
        \includegraphics[width=0.6\textwidth]{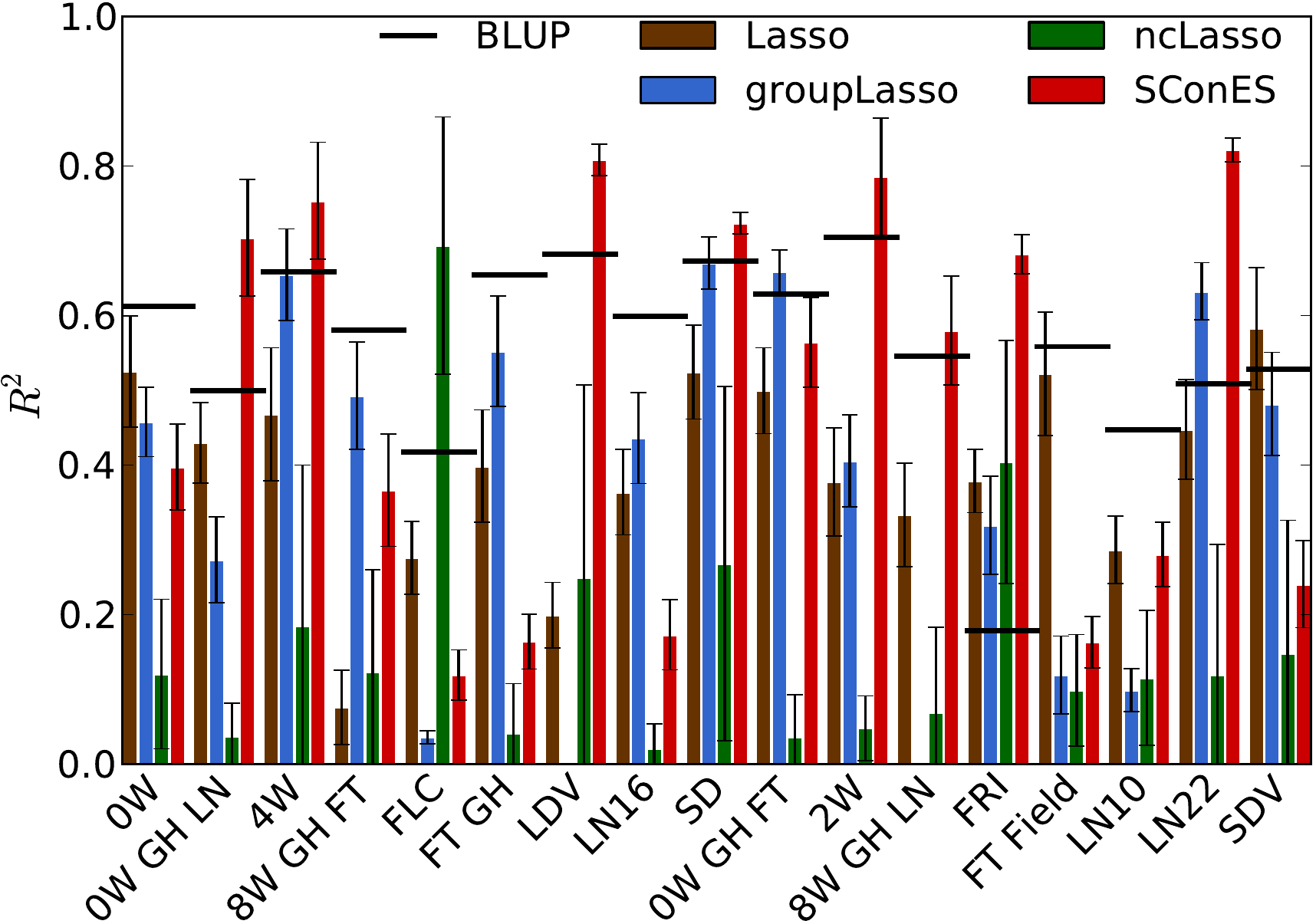}
        \label{fig:ft-pearson-gm}
      }}
    \centerline{\subfloat[Gene interaction networks] {
        \includegraphics[width=0.6\textwidth]{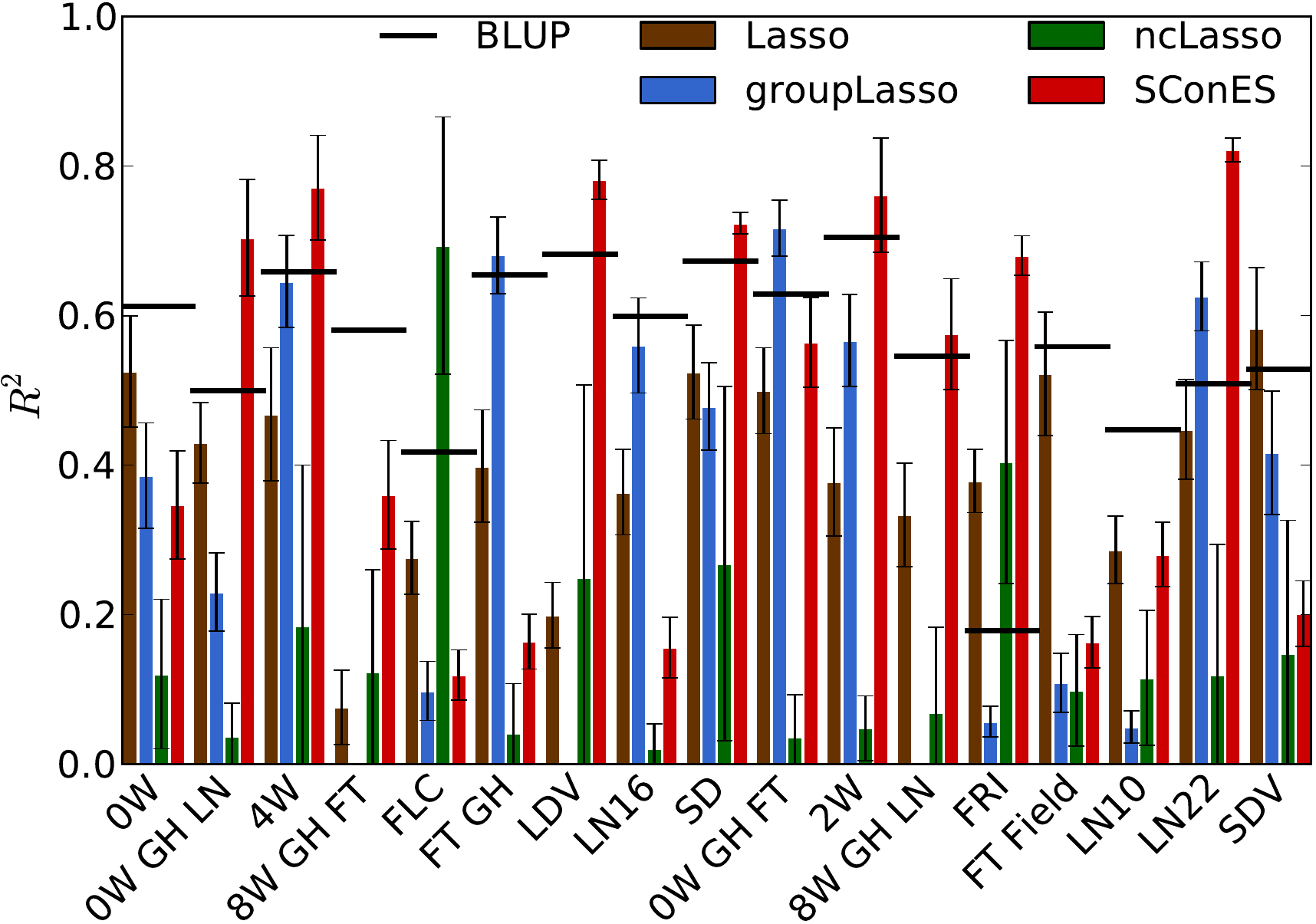}
        \label{fig:ft-pearson-gi}
      }}
    \caption{Cross-validated predictivity (measured as Pearson's
      squared correlation coefficient between actual phenotype and
      phenotype predicted by a ridge-regression over the selected
      SNPs) of \scones\ compared to that of Lasso, groupLasso, and
      ncLasso. Horizontal bars indicate cross-validated BLUP
      predictivity.}
    \label{fig:ft-pearson}
\end{figure*}

\section{Discussion and Conclusions}
\label{sec:conclusions}
In this article, we defined \scones, a novel approach to multi-locus
mapping that selects SNPs that tend to be connected in a given
biological network without restricting the search to predefined sets
of loci.  As the optimization of \scones\ can be solved by maximum
flow, our solution is computationally efficient and scales to whole
genome data. Our experiments show that our method is one to two orders
of magnitude faster than the state-of-the-art Lasso-based comparison
partners, and can therefore easily scale to hundreds of thousands of
SNPs. In simulations, \scones\ is better at leveraging the structure
of the biological network to recover causal SNPs.

On real GWAS data from \textit{Arabidopsis thaliana}, the predictive
ability of the features selected by \scones\ is superior to that of
groupLasso on two of the three network types we consider.  When using
more biological information (gene membership or interactions),
\scones\ tends to recover more distinct explanatory genes than
groupLasso, resulting in better phenotypic prediction.

The constraints imposed by groupLasso and \scones\ are different:
while the groups given to groupLasso and the networks passed to
\scones\ come from the same information, the groups force many more
SNPs to be selected simultaneously when they may not bring much more
information.  This gives \scones\ more flexibility, and makes it less
vulnerable to ill-defined groups or networks, which is especially
desirable in the light of the current noisiness and incompletedness of
biological networks.  Our results on the genomic sequence network
actually indicate that graphLasso, using pairs of network edges as
groups, may achieve the same flexibility as \scones; unfortunately it
is too computationally demanding to be run on the most informative
networks.

We currently derive the SNP networks from neighborhood along the
genome sequence, closeness to a same gene, or proximity to interacting
proteins. Refining those networks and exploring other types of
networks as well as understanding the effects of their topology and
density is one of our next projects.

Let us note that while we do not explicitly consider linkage
disequilibrium, the $l_0$ sparsity constraint of \scones\ should
enforce that when several correlated SNPs are associated with a
phenotype, a single one of them is picked. On the other hand, if
\scones\ is given a genomic sequence network such as the one we
describe, the graph smoothness constraint will encourage nearby SNPs
to be selected together, leading to the selection of sub sequences
that are likely to be haplotype blocks. Such a network should
therefore only be used when the goal of the experiment is to detect
consecutive sequences of associated SNPs.

For now \scones\ considers an additive model between genetic
loci. Future work includes taking pairwise multiplicative effects into
account. Replacing the association term in Equation~\eqref{eq:pb1} by
a sum over pairs of SNPs rather than over individual SNPs results in a
maximum flow problem over a fully connected network of SNPs, which
cannot be solved straightforwardly, if only because the resulting
adjacency matrix is too large to fit in memory on a regular
computer. It might be possible, however, to leverage some of the
techniques used for two-locus GWAS~\citep{achlioptas11,kamthong12} to
help solve this problem.

Extensions of \scones\ to other models include the use of mixed models
to account for population structure and other confounders. This is
currently a challenge as it is unclear how to derive additive test
statistics from such models.

An interesting extension to study would replace the Laplacian by a
random-walk based matrix, derived from powers of the adjacency matrix,
so as to treat differently disconnected SNPs that are closeby in the
networks from those that are far apart. Although we already observe
that SConES is robust to edge removal, this would likely make it more
resistant to missing edges.

Another important extension of \scones\ is to devise a way to evaluate
the statistical significance of the set of selected SNPs.  Regularized
feature selection approaches such as SConES or its Lasso comparison
partners do not lend themselves well to the computation of $p$-values.
Permutation tests could be an option, but the number of permutations
to run is difficult to evaluate, as is that of hypotheses tested.
Another possibility would be to implement the multiple-sample
splitting approach proposed by~\citet{meinshausen09}. However, the
loss of power from performing selection on only subsets of the samples
is too large, given the sizes of current genomic data sets, to make
this feasible. Therefore evaluating statistical significance and
controlling false discovery rates of Lasso and SConES approaches alike
remain a challenge for the future.

Finally, further exciting research topics include applying \scones\ to
larger data sets from human disease consortia (we estimate it would
require less than a day to run on a million of SNPs), and extending it
to the detection of shared networks of markers between multiple
phenotypes.

\section*{Acknowledgments.}
The authors would like to thank Recep Colak, Barbara Rakitsch and Nino
Shervashidze for fruitful discussions.
\paragraph{Funding: } C.A. is funded by an Alexander von
Humboldt fellowship.

\bibliographystyle{natbib} \bibliography{literature}

\end{document}